%% file: main.tex
\documentclass{article}

\usepackage[final]{nips_2018}
\usepackage{nips_2018}
\usepackage{hyperref}
\usepackage{microtype}
\usepackage{graphicx}
\usepackage{subfigure}
\usepackage{booktabs}

\usepackage{times}
\usepackage{amsmath,amssymb,amsthm,amsfonts,latexsym,bbm,xspace,graphicx,float,mathtools,dsfont, xcolor,wrapfig}

\usepackage{algorithm}
\usepackage{algorithmic}

\newcommand{\notshow}[1]{{}}

\newtheorem{definition}{Definition}
\newtheorem{theorem}{Theorem}
\newtheorem{lemma}{Lemma}
\newtheorem{fact}{Fact}

\newenvironment{prevproof}[2]{\noindent {\em {Proof of {#1}~\ref{#2}:}}}{$\Box$\vskip \belowdisplayskip}

\newcommand{\poly}{\text{poly}}
\newcommand{\argmax}{\text{argmax}}
\newcommand{\argmin}{\text{argmin}}

\newcommand{\E}{\mathbb{E}}
\newcommand{\cS}{\mathcal{S}}

\newcommand{\ALG}{\textsc{ALG}}
\newcommand{\WALG}{\widehat{\textsc{ALG}}}

\newcommand{\bR}{\mathbb{R}}
\newcommand{\bS}{\mathbb{S}}
\definecolor{MyGray}{rgb}{0.8,0.8,0.8}

\usepackage[utf8]{inputenc} 
\usepackage[T1]{fontenc}    
\usepackage{hyperref}       
\usepackage{url}            
\usepackage{booktabs}       
\usepackage{amsfonts}       
\usepackage{nicefrac}       
\usepackage{microtype}      

\title{Learning safe policies with expert guidance}

\author{Jessie Huang \qquad Fa Wu \qquad Doina Precup \qquad Yang Cai
	\\School of Computer Science, McGill University
	\\\texttt{\{jiexi.huang,fa.wu2\}@mcgill.ca, \{dprecup,cai\}@cs.mcgill.ca} }
\notshow{
\author{Jessie Huang\\
\texttt{jiexi.huang@mcgill.ca} \\
\And Fa Wu\\
 \texttt{fa.wu2@mail.mcgill.ca} \\
\AND Doina Precup\\
 \texttt{dprecup@cs.mcgill.ca} \\
\And Yang Cai \\
  \texttt{cai@cs.mcgill.ca} \\
}
  }

\begin{document}

\maketitle

\begin{abstract}
\input{abstract}
\end{abstract}
\input{intro}
\input{prelim}
\input{reward_SO}

\input{model_free_exact_solver}
\input{Experiments}

\input{conclusion}
\newpage
\bibliographystyle{plain}
\bibliography{Yang.bib}
\newpage
\input{appendix}
\end{document}

%% file: abstract.tex
We propose a framework for ensuring safe behavior of a reinforcement learning agent when the reward function may be difficult to specify. In order to do this, we rely on the existence of demonstrations from expert policies, and we provide a theoretical framework for the agent to optimize in the space of rewards consistent with its existing knowledge.  We propose two methods to solve the resulting optimization: an exact ellipsoid-based method and a method in the spirit of the "follow-the-perturbed-leader" algorithm. Our experiments demonstrate the behavior of our algorithm in both discrete and continuous problems. The trained agent safely avoids states with potential negative effects while imitating the behavior of the expert in the other states. 

%% file: intro.tex
\section{Introduction}

In Reinforcement Learning (RL), agent behavior is driven by an objective function defined through the specification of rewards. Misspecified rewards may lead to {\it negative side effects} \cite{AmodeiOSCSM16}, when the agent acts unpredictably responding to the aspects of the environment that the designer overlooked, and potentially causes harms to the environment or itself. As the environment gets richer and more complex, it becomes more challenging to specify and balance rewards for every one of its aspects. Yet if we want to have some type of safety guarantees in terms of the behavior of an agent learned by RL once it is deployed in the real world, it is crucial to have a learning algorithm that is robust to mis-specifications.

We assume that the agent has some knowledge about the reward function either through past experience or demonstrations from experts. The goal is to choose a robust/safe policy that achieves high reward with respect to any reward function that is consistent with the agent's knowledge~\footnote{{Note that the safety as used here is more in the context of AI safety, and a policy is safe because it is robust to misspecified rewards and the consequent negative side effects. }}.We formulate this as a \emph{maxmin learning} problem where the agent chooses a policy and an adversary chooses a reward function that is consistent with the agent's current knowledge and minimizes the agent's reward. The goal of the agent is to learn a policy that maximizes the worst possible reward.

 We assume that the reward functions are linear in some feature space. Our formulation has two appealing properties: (1) it allows us to combine demonstrations from multiple experts even though they may disagree with each other; and (2) the training environment/MDP in which the experts operate need not be the same as the testing environment/MDP where the agent will be deployed, our results hold as long as the testing and training MDPs share the same feature space. As an application, our algorithm can learn a maxmin robust policy in a new environment that contains a few features that are not present in the training environment. See our gridworld experiment in Section~\ref{sec:exp_gw}. 
 
 Our first result (Theorem~\ref{thm:exact maxmin}) shows that given any algorithm that can find the optimal policy for an MDP in polynomial time, we can solve the maxmin learning problem exactly in polynomial time. Our algorithm is based on a seminal result from combinatorial optimization -- the equivalence between separation and optimization~\cite{GrotschelLS81,KarpP82} -- and the ellipsoid method. To understand the difficulty of our problem, {it is useful to think of maxmin learning as a two-player zero-sum game between the agent and the adversary.} The deterministic policies correspond to the pure strategies of the agent. The consistent reward functions we define in Section~\ref{sec:reward SO} form a convex set and the adversary's pure strategies are the extreme points of this convex set. Unfortunately, both the agent and the adversary may have exponentially many pure strategies, which are hard to describe explicitly. This makes solving the two-player zero-sum game challenging.  Using tools from combinatorial optimization, we manage to construct separation oracles for both the agent's and the adversary's set of policies using the MDP solver as a subroutine. With the separation oracles, we can solve the maxmin learning problem in polynomial time using the ellipsoid method.

 
  Theorem~\ref{thm:exact maxmin} provides a polynomial time algorithm, but as it heavily relies on the ellipsoid method, it is computationally expensive to run in practice. We propose another algorithm (Algorithm~\ref{alg:FPL maxmin}) based on the online learning algorithm -- \emph{followed-the-perturbed-leader (FPL)}, and show that after $T$ iterations the algorithm computes a policy that is at most $O\left(\nicefrac{1}{\sqrt{T}}\right)$ away from the true maxmin policy (Theorem~\ref{thm:FPL maxmin}). Moreover, each iteration of our algorithm is polynomial time. Notice that many other low-regret learning algorithms, such as the multiplicative weights update method (MWU), are not suitable for our problem. The MWU requires explicitly maintaining a weight for every pure strategy 
   and updates them in every iteration, resulting in an exponential time algorithm for our problem. Furthermore, we show that Algorithm~\ref{alg:FPL maxmin} still has similar performance when we only have a fully polynomial time approximation scheme (FPTAS) for solving the MDP. The formal statement and proof are postponed to the supplemental material due to space limit.

\subsection{Related Work}\label{sec:related}
In the sense of using expert demonstrations, our work is related to inverse reinforcement learning (IRL) and apprenticeship learning~\cite{NgR00,AbbeelN04,SyedS08,SyedBS08}. In particular, the apprenticeship learning problem {aims to learn a policy that performs at least as well as the expert's policy on all basis rewards}, and can also be formulated as a maxmin problem~\cite{SyedS08,SyedBS08}. Despite the seemingly similarity, our maxmin learning problem aims to solve a completely different problem than apprenticeship learning. Here is a simple example: {consider in a gridworld, there are two basis rewards, $w_1$ and $w_2$, and there are only two routes/policies -- $top$ and $bottom$. The expert takes the $top$ route, getting $100$ under $w_1$ and $70$ under $w_2$. Alternatively, taking the $bottom$ route gets $90$ under both $w_1$ and $w_2$. Apprenticeship learning will return the { $top$ route}, because taking the alternative route performs worse than the expert under $w_1$ and violates the requirement. What is our solution? Assuming $\epsilon=25$~\footnote{See Section~\ref{sec:reward SO} for the formal definition of consistent rewards. Intuitively, it means that the expert's policy yields a reward that is within $\epsilon$ of the optimal possible reward.}, both $w_1$ and $w_2$ (or any convex combination of them) are consistent with the expert demonstration. If we choose the top route, the worst case performance (under $w_2$ in this case) is 70, while the correct maxmin solution to our problem is to take the {$bottom$ route} so that its worst performance is $90$.}\notshow{assume the expert is consistent with two reward functions $w_1$ and $w_2$ with $\epsilon = 25$~\footnote{See Section~\ref{sec:reward SO} for the formal definition of consistent rewards. Intuitively, it means that the expert's policy yields a reward that is within $\epsilon$ of the optimal possible reward.}. The expert gets $100$ under $w_1$ and $70$ under $w_2$, and the only alternative policy gives $90$ and $90$. Apprenticeship learning will return the expert policy while ours will return the latter solution.} In the worst case (under $w_2$), our maxmin policy has better guarantees and thus is more robust. 
Unlike apprenticeship learning/IRL, we do not want to mimic the experts or infer their rewards, but we want to produce a policy with robustness guarantees by leveraging their data. As a consequence, our results are applicable to settings where the training and testing environments are different (as discussed in the Introduction). Moreover, our formulation allows us to combine multiple expert demonstrations.  

\notshow{\paragraph{Apprenticeship Learning} Similarity: maxmin formulation, expert guidance. Difference: apprenticeship hopes to imitate the expert's behavior. We just want to infer consistent rewards from the expert and then do maxmin planning. Give example. Different application: we can accommodate multiple experts and do not need the training domain to be the same as testing domain. See experiment.~\cite{SyedBS08,SyedS08}}

Inverse reward design~\cite{HadfieldMARD17}  uses a proxy reward and infers the true reward by estimating its posterior. Then it uses risk-averse planning together with samples from the posterior in the testing environment to achieve safe exploration. Our approach achieves a similar goal without assuming any distribution over the rewards and is arguably more robust. We apply a single reward function to the whole MDP while they apply (maybe too pessimistically) per step/trajectory maxmin planning. Furthermore, our algorithm is guaranteed to find the maxmin solution in polynomial time, and can naturally accommodate multiple experts. 

{In repeated IRL~\cite{AminJS17}, the agent acts on the behalf of a human expert in a variety of tasks, and the human expert corrects the agent when the agent's policy is far from the optimum. The goal is to minimize the number of corrections from the expert, and they provide an upper bound on the number of corrections by reducing the problem to the ellipsoid method. Their model requires continuous interaction with an expert while our model only assumes the availability of one or a couple expert policies prior to training. Furthermore, we aim to find a maxmin optimal policy, while their paper focuses on minimizing the number of corrections needed.}

Robust Markov Decision Processes  \cite{NilimEG05,Iyengar05}  have addressed the problem of performing dynamic programming-style optimization environments in which the transition probability matrix is uncertain. Lim, Xu \& Mannor \cite{LimXM13} have extended this idea to reinforcement learning methods.  This body of work also uses min-max optimization, but because the optimization is with respect to worst-case transitions, this line of work results in very pessimistic policies.  Our algorithmic approach and flavor of results are also different.
\cite{MorimotoD05} have addressed a similar adversarial setup, but in which the environment designs a worst-case disturbance to the dynamics of the agent, and have addressed this setup using $H_\infty$ control. 


\paragraph{Paper Organization:} We introduce the notations and  define the maxmin learning problem in Section~\ref{sec:prelim}. 
We provide three different ways to  
define the set of consistent reward functions in Section~\ref{sec:reward SO}, and present the ellipsoid-based exact algorithm and its analysis in Section~\ref{sec:exact alg}. The FPL-based algorithm and its analysis are in Section~\ref{sec:exact FPL}, followed by experimental results in Section~\ref{sec:exp.}.

%% file: prelim.tex
\section{Preliminary}\label{sec:prelim}

An MDP is a tuple $M = (\mathcal{S}, \mathcal{A}, P_{sa}, \gamma, D, R)$, including a finite set of states, $\mathcal{S}$, a set of actions, $\mathcal{A}$, and transition probabilities, $P_{sa}$.  $\gamma$ is a discount factor, and $D$ is the distribution of initial states.  The reward function $R$ instructs the learning process. We assume that the reward is a linear function of some vector of features $\phi$: $ \mathcal{S} \to  [0, 1] ^ k$  over states. That is $R(s) = w \cdot \phi(s)$ for every state $s\in \cS$, where $w\in\mathbb{R}^k$ is the \emph{reward weights} of the MDP. The true reward weights $w^*$ is assumed to be unknown to the agent. We use $\langle\cdot\rangle$ to denote the bit complexity of an object. In particular, we use $\langle M\rangle$ to denote the bit complexity of $M$, which is the number of bits required to represent the distribution of initial states, transition probabilities, the discount factor $\gamma$, and the rewards at all the states. We use the notation $M\backslash R$ to denote a MDP without the reward function, and $\langle M\backslash R \rangle$ is its bit complexity. We further assume that $\phi(s)$ can be represented using at most $\langle \phi \rangle$ bits for any state $s\in \mathcal{S}$.

An agent selects the action according to a policy $\pi$. The value of a policy under rewards $w$ is $ \E_{s_0 \sim D}[V^\pi(s_0) | M] = w \cdot \E[\sum_{t=0}^{\infty} \gamma^t\phi(s_t) | M,\pi]$. It is expressed as the weights multiplied by the accumulated discounted feature value given a policy, which we define as $\Psi(\pi)=\E[\sum_{t=0}^{\infty} \gamma^t\phi(s_t) | M,\pi]$.

\paragraph{MDP solver}
We assume that there is a RL algorithm $\ALG$ that takes an MDP as input and outputs an optimal policy and its corresponding representation in the feature space. In particular, $\ALG(M)$ outputs $(\pi^*,\mu^*)$ such that $\E_{s_0\sim D}[V^{\pi^*}(s_0) | M]=\max_\pi  \E_{s_0 \sim D}[V^\pi(s_0) | M]$ and $\mu^*=\Psi(\pi^*)$.

\paragraph{Maxmin Learning}
All weights that are consistent with the agent's knowledge form a set $P_R$. We will discuss several formal ways to define this set in Section~\ref{sec:reward SO}. The goal of the agent is to learn a policy that maximizes the reward for any reward function that could be induced by weights in $P_R$ and adversarially chosen. More specifically, the max-min learning problem is $\boldsymbol{\mathop{\max}_{\mu\in P_F} \mathop{\min}_{w\in P_R} w^T\mu}$,
where $P_F$ is the polytope that contains the representations of all policies in the feature space, i.e. $P_F=\{\mu\ | \ \mu =\Psi(\pi) \text{ for some policy $\pi$ } \}$. WLOG, we assume that all weights lie in $[-1,1]^k$.



\notshow{\subsection {Ellipsoid}

The Ellipsoid method can solve an LP in polynomial time. Consider an input LP that has the form of:

\begin {equation}
K = { Ax \leq b : x \geq 0}, 
\end{equation}

where $A \subseteq \bR ^ n$. If $n=1$, we can solve the LP directly. Without loosing generality, we asume $n>1$. The Ellipsoid algorithm maintains an ellipsoid $E$ that completely contains the convex set $K$, if $K$ is nonempty.  We check if the center of the ellipsoid is in $K$. If it is, we are done and we find a point in K. If it is not, we find a hyperplane $a^Tx \leq b$ that is satisfied by all points in $K$ but not the center of the ellipsoid. The Ellipsoid algorithm is as follows:

\begin{algorithm}[ht]
\begin{algorithmic}[1]
\REQUIRE {An input set $K \subseteq \bR^n$.}
\STATE Initialize an ellipsoid $E_o$ that contains $K$, $k = 0$.
\IF {the center $c_k$ of $E_k$ is in $K$}
	\STATE Terminate and report $c_k$
\ELSE \STATE Get a separating hyperplane through $c_k$, i.e., $a^T x = a^T c_k $ such that the set K is completely contained in the half-ellipsoid formed by the intersection of the half space and the current ellipsoid. 
	\STATE Take the smallest volume ellipsoid containing the half-ellipsoid that contains $K$, set $k=k+1$ and repeat the iteration
\ENDIF
\end{algorithmic}
\caption{{\sf Ellipsoid Algorithm}}
\label{alg:Ellipsoid}
\end{algorithm}

The ellipsoids constructed by the algorithm shrink in volume as the algorithm proceeds. In the end, we either find a point in the input set (solve the LP) or determine that the input set is empty. The algorithm is polynomial-time. In order to use the ellipsoid method, to solve LPs in our problem, we need separation oracles as access to the polytopes of interest, which we discuss later.}

\paragraph{Separation Oracles} To perform maxmin learning, we often need to optimize linear functions over convex sets that are intersections of exponentially many halfspaces. Such optimization problem is usually intractable, but if the convex set permits a polynomial time \emph{separation oracle}, then there exists polynomial time algorithms (e.g. ellipsoid method) that optimize linear functions over it.

\begin{definition}(\textbf{Separation Oracle}) Let $P$ be a closed, convex subset of Euclidean space $\mathbb{R}^d$. Then a \emph{Separation Oracle} for $P$ is an algorithm that takes as input a point ${x}\in\mathbb{R}^d$ and outputs ``\textsc{Yes}'' if ${x} \in P$, or a hyperplane $({w},c)$ such that $w\cdot {y} \leq c$ for all ${y} \in P$, but $w\cdot{x} > c$. Note that because $P$ is closed and convex, such a hyperplane always exists whenever ${x} \notin P$. 


\end{definition}

%% file: reward_SO.tex
\section{Consistent Reward Polytope}\label{sec:reward SO}
In this section, we discuss several ways to define the consistent reward polytope $P_R$.

\paragraph{Explicit Description} We assume that the agent knows that the weights satisfy a set of explicitly defined linear inequalities of the form $c\cdot w \geq b$. For example, such an inequality can be learned by observing that a particular policy yields a reward that is larger or smaller than a certain threshold.~\footnote{Note that with a polynomial number of trajectories, one can apply standard Chernoff bounds to derive such inequalities that hold with high probability. It is often the case that the probability is so close to $1$ that the inequality can be treated as true always for any practical purposes.}

\paragraph{Implicitly Specified by an Expert Policy} Usually, it may not be easy to obtain many explicit inequalities about the weights. Instead, we may have observed a policy $\pi_E$ used by an expert. We further assume that the expert's policy has a reasonably good performance under the true rewards $w^*$. Namely, $\pi_E$'s expected reward is only $\epsilon$ less than the optimal one. Let the expert's feature vector $\mu_E = \Psi(\pi_E)$. The set $P_R$ therefore contains all $w$ such that $\mu_E \cdot w \geq \mu^T \cdot w - \epsilon, \forall \mu \in P_F$. It is not hard to verify that under this definition $P_R$ is a convex set. Even though explicitly specifying $P_R$ is extremely expensive as there are infinitely many $\mu\in P_F$, we can construct a polynomial time separation oracle $SO_R$ (Algorithm~\ref{alg:SO}). An alternative way to define $P_R$ is to assume that the expert policy can achieve $(1-\epsilon)$ of the optimal reward (assuming the final reward is positive). We can again design a polynomial time separation oracle similar to Algorithm~\ref{alg:SO}. 

		\begin{algorithm}[H]
			\begin{algorithmic}[1]
				\INPUT {$w' \in \mathbb{R}^{k}$}
				\STATE Let $\mu_{w'}:=\mathop{\argmax}_{\mu\in P_F} \mu \cdot w' $. Notice that $\mu_{w'}$ is the feature vector of the optimal policy under reward weights $w'$. Hence, it can be computed by our MDP solver $\ALG$.
				\IF {$\mu_{w'}\cdot w'> \mu_E \cdot w' + \epsilon$} \STATE output ``NO'' , and 
	 				$ \left(\mu_E- \mu_{w'} \right)\cdot w + \epsilon\geq 0$ as the separating hyperplane, 
	 				since for all $w \in P_R, \mu_E\cdot w \geq \mu_{w'} \cdot w -\epsilon$.
				\ELSE \STATE output ``YES''.
				\ENDIF
			\end{algorithmic}
		\caption{{\sf Separation Oracle $SO_R$ for the reward polytope $P_R$}}
		\label{alg:SO}
		\end{algorithm}

\paragraph{Combining Multiple Experts} 
How can we combine demonstrations from experts operating in drastically different environments? Here is our model. For each environment $i$, there is a separate MDP $M_i$, and all the MDPs share the same underlying weights as they are all about completing the same task although in different environments. The $i$-th expert's policy is nearly optimal in $M_i$. More specifically, for expert $i$, her policy $\pi_{E_i}$ is at most $\epsilon_i$ less than the optimal policy in $M_i$. Therefore, each expert $i$ provides a set of constraints that any consistent reward needs to satisfy, and $P_R$ is the set of rewards that satisfy all constraints imposed by the experts. For each expert $i$, we can design a separation oracle $SO^{(i)}_R$ (similar to Algorithm~\ref{alg:SO}) accepting weights that respect the constraints given by expert $i$'s policy. We can easily design a separation oracle for $P_R$ that only accepts weights that will be accepted by all separation oracles $SO^{(i)}_R$. 

From now on, we will not distinguish between different ways to define and access the consistent reward polytope $P_R$, but simply assume that we have a polynomial time separation oracle for it. All the algorithms we design in this paper only require access to this separation oracle. In Section~\ref{sec:exp.}, we will specify how the $P_R$ is defined for each experiment. 

%% file: model_free_exact_solver.tex
\section{Maxmin Learning using an Exact MDP Solver}\label{sec:alg}
In this section, we show how to design maxmin learning algorithms. Our algorithm only interacts with the MDP through the MDP solver, which can be either model-based or model-free. Our first algorithm solves the maxmin learning problem exactly using the ellipsoid method. Despite the fact that the ellipsoid method has provable worst-case polynomial running time, it is known to be inefficient sometimes in practice. Our second algorithm is an efficient iterative method based on the online learning algorithm -- \emph{follow-the-perturbed-leader (FPL)}. 

\subsection{Ellipsoid-Method-Based Solution}\label{sec:exact alg}
\begin{theorem}\label{thm:exact maxmin}
	Given a polynomial time separation oracle $SO_R$ for the consistent reward polytope $P_R$ and an exact polynomial time MDP solver $\ALG$, we have a polynomial time algorithm such that for any MDP without the reward function $M\backslash R$, the algorithm computes the maxmin policy $\pi^*$ with respect to $M\backslash R$ and $P_R$. \end{theorem}
	

\begin{wrapfigure}{r}{0.5\textwidth}
\centering
\colorbox{MyGray}{
\begin{minipage}{0.45\textwidth} 
\vspace{-10pt}
\begin{equation*}
\begin{split}
\max\quad \displaystyle &z \\
 \text{subject to}\quad \displaystyle &z \leq \mu \cdot w, ~~~ \forall w \in P_R\\
&\mu \in P_F
\end{split}	
\end{equation*}
\end{minipage}}
\caption{Maxmin Learning LP.}
\label{fig:maxmin LP}
\vspace{-10pt}
\end{wrapfigure}
The plan is to first solve the maxmin learning problem in the feature space then convert it back to the policy space. Solving the maxmin learning problem in the feature space is equivalent to solving the linear program in Figure~\ref{fig:maxmin LP}.

The challenges for solving the LP are that (i) it is not clear how to check whether $\mu$ lies in the polytope $P_F$, and (ii) there are seemingly infinitely many constraints of the type $z\leq \mu\cdot w$ as there are infinitely many $w\in P_R$. Next, we show that given an exact MDP solver $\ALG$, we can design a polynomial time separation oracle for the set of feasible variables $(\mu,z)$ of LP~\ref{fig:maxmin LP}. With this separation oracle, we can apply the ellipsoid method (see Theorem~\ref{thm:ellipsoid} in the supplementary material) to solve LP~\ref{fig:maxmin LP} in polynomial time.

First, we design a separation oracle for polytope $P_F$ by invoking a seminal result from optimization -- the equivalence between separation and optimization.

\begin{lemma}[Separation $\equiv$ Optimization]~\cite{GrotschelLS81,KarpP82}\label{lem:sep-opt}
Consider any convex polytope $P = \{x : Ax \leq b\}\in \mathbb{R}^d$ and the following two problems:
\begin{itemize}
	\item \textbf{Linear Optimization:} given a linear objective $c \in \mathbb{R}^d$, compute $x^* \in \argmax_{x\in P} c\cdot x$
	\item \textbf{Separation:} given a point $y \in\mathbb{R}^d$, decide that $y \in P$, or else find $h \in \mathbb{R}^d$ s.t. $h\cdot x < h\cdot y$, $\forall x \in P$.
\end{itemize}
If $P$ can be described implicitly
using $\langle P\rangle$ bits, then the separation problem is solvable in $\poly(\langle P\rangle, d,\langle y\rangle)$ time for P  if and
only if the linear optimization problem is solvable in $\poly(\langle P\rangle, d,\langle c\rangle)$
time.\end{lemma}
	
	 It is not hard to see that if one can solve the separation problem, one can construct a separation oracle in polynomial time and apply the ellipsoid method to solve the linear optimization problem. The less obvious direction in the result above states that if one can solve the linear optimization problem, one can also use it to construct a separation oracle. The equivalence between these two problems  turns out to have profound implications in combinatorial optimization and has enabled numerous polynomial time algorithms for many problems that are difficult to solve otherwise.
	\begin{algorithm}[H]
		\begin{algorithmic}[1]
			\INPUT $(\mu',z')\in\mathbb{R}^{k+1}$
			\STATE Query $SO_F(\mu')$.
			\IF{$\mu' \notin P_F$} \STATE{output ``\textsc{No}'' and output the same separating hyperplane as outputted by $SO_F(\mu')$.}
			\ELSE\STATE Let $w^* \in \mathop{\text{argmin}}_{w \in P_R} \mu' \cdot w$ and $V=\mu' \cdot w^*$. This requires solving a linear optimization problem over $P_R$ using the ellipsoid method with the separation oracle $SO_R$.

			\IF{$z' \leq V$}\STATE {output ``\textsc{Yes}''}
			\ELSE \STATE output ``\textsc{No}'', and a separating hyperplane $z \leq \mu \cdot w^*$, as $z'> \mu'\cdot w^*$ and all feasible solutions of LP~\ref{fig:maxmin LP} respect this constraint.
			\ENDIF
			\ENDIF
		\end{algorithmic}
		\caption{{\sf Separation Oracle for the feasible $(\mu, z)$ in LP~\ref{fig:maxmin LP}}}
		\label{alg:SO maxmin}
	\end{algorithm}

Our goal is to design a polynomial time separation oracle for the polytope $P_{F}$. The key observation is that the linear optimization problem over polytope $P_{F}$: $\max_{\mu\in P_F} w\cdot \mu$ is exactly the same as solving the MDP with reward function $R(\cdot)=w\cdot \phi(\cdot)$. Therefore, we can use the MDP solver to design a polynomial time separation oracle for $P_F$.

\begin{lemma}\label{lem:SO_f}
	Given access to an MDP solver $\ALG$ that solves any MDP $M$ in time polynomial in $\langle M\rangle$, we can design a separation oracle $SO_F$ for $P_F$ that runs in time polynomial in $\langle M\backslash R \rangle$, $\langle \phi \rangle$, $k$, and the bit complexity of the input~\footnote{Note that  $SO_F$ only depends on the bit complexity of $M\backslash R$, but not the actual model of $M\backslash R$ such as the distributions of the initial states or  the transition probabilities. We only require access to $ALG$ and an upper bound of $\langle M\backslash R \rangle$.}.
\end{lemma}


\notshow{		\begin{figure}[h!]
	\colorbox{MyGray}{
	\begin{minipage}{0.46\textwidth} {	
		$WSO(\vec{\pi})=$
	\begin{itemize}	
	\item ``\textbf{Yes}'' if the ellipsoid algorithm with $N$ iterations\footnote{The appropriate choice of $N$ for our use of $WSO$ is provided in Corollary~\ref{cor:N} of Section~\ref{sec:runtime}. The only place that requires an appropriate choice of $N$ is the proof of Lemma~\ref{lem:convhull}.} outputs ``infeasible'' on the following problem:
			 
			  \underline{\textbf{variables:}} $\vec{w}, t$;
			 
			 \underline{\textbf{constraints:}}
				\begin{itemize}
		 		\item $\vec{w} \in [-1,1]^d;$
		 		\item $t - \vec{\pi} \cdot \vec{w} \leq -\delta$;\footnote{The appropriate choice of $\delta$ for our use of $WSO$ is provided in Lemma~\ref{lem:delta} of Section~\ref{sec:runtime}. The only place that requires an appropriate choice of $\delta$ is the proof of Lemma~\ref{lem:convhull}.}
		 		\item $\widehat{WSO}(\vec{w},t) = $
		 		\begin{itemize}
		 		\item ``yes'' if $t \geq \mathcal{A}(\vec{w}) \cdot \vec{w}$;\footnote{Notice that the set $\{(\vec{w},t)|\widehat{WSO}(\vec{w},t) =$ ``Yes''$\}$ is not necessarily convex or even connected.}
		 		\item the violated hyperplane $t' \geq \mathcal{A}(\vec{w})\cdot \vec{w}'$ otherwise.
		 		\end{itemize}
				 		\end{itemize}
	\item If a feasible point $(t^*,\vec{w}^*)$ is found, output the violated hyperplane $\vec{w}^* \cdot \vec{\pi}' \leq t^*$.
	\end{itemize}		}
			\end{minipage}} \caption{A ``weird'' separation oracle.}\label{fig:WSO}
	\end{figure}
}



With $SO_F$, we first design a polynomial time separation oracle for checking the feasible $(z,\mu)$ pairs in LP~\ref{fig:maxmin LP} (Algorithm~\ref{alg:SO maxmin}). With the separation oracle, we can solve LP~\ref{fig:maxmin LP} using the ellipsoid method. The last difficulty is that the optimal solution only gives us the maxmin feature vector instead of the corresponding maxmin policy. We use the following nice property of $SO_F$ to convert the optimal solution in the feature space to the policy space. See Section~\ref{sec:missingproof} in the supplementary material for intuition behind Lemma~\ref{lem:convhull}.

\begin{lemma}~\cite{GrotschelLS81,KarpP82,CaiDW13a}\label{lem:convhull}
 If $SO_F(\mu) = $ ``\textsc{Yes}'', there exists a set, $C$, of weights ${w}\in \mathbb{R}^k$ such that $SO_F$ has queried the MDP solver $\ALG$ on reward function $w\cdot \phi(\cdot)$ for every $w\in C$. Let $(\pi_w,\mu_w)$ be the output of $\ALG$ on weight $w$, then $\mu$ lies in the convex hull of  $\{\mu_w | {w} \in C\}$. 
	\end{lemma}
	


\notshow{\begin{enumerate}
\item Check if $\mu' \in P_F$ by querying the separation oracle of $P_F$, if not, directly output ``NO'' and output the same separating hyperplane as output by $P_F$.
\item Let $w^* = \mathop{\text{argmin}}_{w \in P_R} \mu' \cdot w$ and $V=\mu' \cdot w^*$. 
\item If $z' \leq V$, output ``YES'';
\item Else output ``NO'', and a separating hyperplane $z \leq \mu \cdot w^*$.
\end{enumerate}}

\begin{prevproof}{Theorem}{thm:exact maxmin}
It is not hard to see that Algorithm~\ref{alg:SO maxmin} is a valid polynomial time separation oracle for the feasible $(\mu,z)$ pairs in LP~\ref{fig:maxmin LP}. Hence, we can solve LP~\ref{fig:maxmin LP} in polynomial time with the ellipsoid method with access to Algorithm~\ref{alg:SO maxmin}. Next, we show how to convert the optimal solution $\mu^*$ of LP~\ref{fig:maxmin LP} to the corresponding maxmin optimal policy $\pi^*$. Here, we invoke Lemma~\ref{lem:convhull}. We query $SO_F$ on $\mu^*$ and we record all weights $w$ that $SO_F$ has queried the MDP solver $\ALG$ on. Let $C=\{w_1,\ldots, w_\ell\}$ be all the queried weights. As $SO_F$ is a polynomial time algorithm, $\ell$ is also polynomial in the input size. By Lemma~\ref{lem:convhull}, we know that $\mu$ is in the convex hull of $\left(\{\mu_w | {w} \in C\}\right)$, which means there exists a set of nonnegative numbers $p_1,\ldots, p_\ell$, such that $\sum_{i=1}^\ell p_i =1$ and $\mu^*=\sum_{i=1}^\ell p_i\cdot \mu_{w_i}$. Clearly, the discounted accumulated feature value of the randomized policy $\sum_{i=1}^\ell p_i\cdot \pi_{w_i}$ equals to $\sum_{i=1}^\ell p_i\cdot \Psi(\pi_{w_i})=\sum_{i=1}^{\ell} p_i\cdot\mu_{w_i}=\mu^*$. We can compute the $p_i$s in poly-time via linear programming and  $\sum_{i=1}^\ell p_i\cdot \pi_{w_i}$ is the maxmin policy.
\end{prevproof}


\subsection {Finding the Maxmin Policy using Follow the Perturbed Leader}\label{sec:exact FPL}
The exact algorithm of Theorem~\ref{thm:exact maxmin} may be computationally expensive to run, as the separation oracle $SO_F$ requires running the ellipsoid method to answer every query, and on top of that we need to run the ellipsoid method with queries to $SO_F$. In this section, we propose a simpler and faster algorithm that is based on the algorithm -- \emph{follow-the-perturbed-leader (FPL)}~\cite{KalaiV05}.

\begin{theorem}\label{thm:FPL maxmin}
	For any $\xi\in (0,1/2)$, with probability at least $1-2\xi$, Algorithm~\ref{alg:FPL maxmin} finds a policy $\pi$ after $T$ rounds of iterations such that its expected reward under any weight from $P_R$ is at least $\max_{\mu\in P_F}\min_{w\in P_R} \mu\cdot w-\frac{k^2\left(6+4\sqrt{\ln 1/\xi}\right)}{\sqrt{T}}$. In every iteration, Algorithm~\ref{alg:FPL maxmin} makes one query to $\ALG$ and $O\left(k^2 \left((\log k)^2+\left((b+\langle\phi \rangle) (|\mathcal{A}||\mathcal{S}|+k)+\log T \right)^2\right)\right)$ queries to $SO_R$, where $b$ is an upper bound on the number of bits needed to specify the transition probability	$P_{sa}$ for any state $s$ and action $a$.
\end{theorem}
FPL is a classical online learning algorithm that solves a problem where a series of decisions $d_1, d_2, ...$ need to be made. Each $d_i$ is from a possibly infinite set $D \subseteq \mathbb{R}^n$. The state $s_t \in \bS \subseteq \mathbb{R}^n$ at step $t$ is observed after the decision $d_t$. The goal is to have the total reward $\sum_t d_t \cdot s_t$ not far from the best expert's reward with hindsight, that is $\max_{d\in D} \sum_t{d \cdot s_t}$. The FPL algorithm guarantees that after $T$ rounds, the regret $\sum_t d_t \cdot s_t-\max_{d\in D} \sum_t{d \cdot s_t}$ scales linearly in $\sqrt{T}$. This guarantee holds for both oblivious and adaptive adversary, and the bound holds both in expectation and with high probability (see Theorem~\ref{thm:FPL} in Section~\ref{sec:missingproof} of the supplementary material for the formal statement).

	\begin{algorithm}[H]
		\begin{algorithmic}[1]
\INPUT $T$: the number of iterations
\STATE Set $\delta:=\nicefrac{1}{k\sqrt{T}}$.
\STATE Arbitrarily pick some policy $\pi_1$, compute $\mu_1 \in P_F$. Arbitrarily pick some reward weights $w_1$, and set $t = 1$.
\WHILE{$t\leq T$}{
\STATE Use $\ALG$ to compute the optimal policy $\pi_t$ and $\mu_t = \Psi(\pi_t)$ that maximizes the expected reward under reward function $ \left(\sum_{i=1}^{t-1} w_i+p_t\right )\cdot \phi(\cdot)$, where $p_t$ is drawn uniformly from $[0,1/\delta]^k$.
\STATE Let $w_t := \mathop{\argmin}_{w\in P_R} w^T (\sum_{i=1}^{t-1}\mu_t + q_t)$, where $q_t$ is drawn uniformly from $[0,1/\delta]^k$.
\STATE $t:= t+1$.}
\ENDWHILE
\STATE Output the randomized policy $\frac{1}{T}\cdot \sum_{t=1}^T \pi_t$.
\end{algorithmic}
\caption{{\sf FPL Maxmin Learning}}
\label{alg:FPL maxmin}
\end{algorithm}

FPL falls into a large class of algorithms that are called low-regret algorithms, as the regret grows sub-linearly in $T$. It is well known that low-regret algorithms can be used to solve two-player zero-sum games approximately. The maxmin problem we face here can also be modeled as a two-player zero-sum games. One player is the agent whose strategy is a policy $\pi$, and the other player is the reward designer whose strategy is a weight $w\in P_R$. The agent's payoff is the reward that it collects using policy $\pi$, which is $\Psi(\pi)\cdot w$, and the designer's payoff is $-\Psi(\pi)\cdot w$. Finding the maxmin strategy for the agent is equivalent to finding the maxmin policy. One challenge here is that the numbers of strategies for both players are infinite. Even if we only consider the pure strategies which correspond to the extreme points of $P_F$ and $P_R$, there are still exponentially many of them. 
Many low-regret algorithms such as multiplicative-weights-update requires explicitly maintaining a distribution over the pure strategies, and update it in every iteration. In our case, these algorithms will take exponential time to finish just a single iteration. This is the reason why we favor the FPL algorithm, as the FPL algorithm only requires finding the best policy giving the past weights, which can be done by the MDP solver $\ALG$. We also show that a similar result holds even if we replace the exact MDP solver with an additive FPTAS $\WALG$.
{The proof of Theorem~\ref{thm:FPL maxmin} can be found in Section~\ref{sec:missingproof} in the supplementary material. Our generalization to cases where we only have access to $\WALG$ is postponed to Section~\ref{sec:approx FPL} in the supplementary material.}  


\notshow{The algorithm works as follows. In the first round, we choose arbitrary $\mu_1 \in P_F$, and $w_1 \in P_R$. Then for every $t > 1$,  choose $\mu_t$ using FPL given all $w_{1:t-1}$. That is, 

\begin{equation}\label{def:FTPL}
\mathop{\text{argmax}}_{\mu \in P_F} \mu\cdot (\mathop{\sum}_{j=1}^{t-1} w_j + p_t), 
\end{equation}

where $p_t$ is drawn from a distribution that is defined by the FPL. Once the agent has chosen the policy and thus the features, the reward weights are chosen adversarially such that $w_t= \text{argmin}_{w \in P_R} w_t^T \mu_t$. The game will go on for $T$ rounds and we claim that the time average $1/T \sum_{t=1}^T \mu_t$ is roughly the max-min solution, and proof is as follows:

\begin{align*}
\mathop {\min}_{w\in P_R} \frac{1}{T} \sum_{t=1}^T w^T \mu_t & \geq  \frac{1}{T} \sum_{t=1}^T w_t^T\mu_t &&\\
& \geq \frac{1}{T}  \mathop{\max}_{\mu \in P_F} \sum_{t=1}^T w_t^T\mu - \text{regret} && (\text {regret bound of FPL})\\
& \geq \mathop{\min}_{w \in P_R} \mathop{\max}_{\mu \in P_F} w^T \mu - \text {regret} && \\
& = \mathop{\max}_{\mu \in P_F} \mathop{\min}_{w \in P_R} w^T \mu - \text{regret}
\end{align*}

The first inequality holds because each $w_t$ was chosen to minimize the reward with respect to each $\mu_t$. Thus the sum of rewards on the right hand side is smaller or equal to the sum of minimal rewards achieved by a common set of weights that is on the left hand side.

The regret bound of FPL is 

\begin{equation}
\sum_{t=1}^T w_t^T\mu_t  \geq  \mathop{\max}_{\mu \in P_F} \sum_{t=1}^T w_t^T\mu - \text{regret} \cdot T,
\end{equation}

and
\begin{equation}
\text{regret} = \sqrt{\frac{RAD}{T}},
\end{equation}

where $R$ is the largest reward possible in the MDP, $A$ is the $L_1$ norm of any possible weights, and  $D$ is the $L_1$ diameter of the feature polytope $P_F$. The (average) regret decreases as the total time step $T$ increases, $\text{regret} \approx O(1/\sqrt{T})$. Thus the time average $\sum_{t=1}^T \mu_t$ is roughly the max-min solution. We can adjust the total time step to achieve the desired accuracy. The perturbation term $p_t$ is chosen randomly according to an exponential distribution independently for each coordinate, that is to choose from a standard exponential distribution $\epsilon_2 \exp(-\epsilon_2 x)$, with $\epsilon_2= \sqrt{D/RAT} $.}

%% file: Experiments.tex
\section{Experiments}\label{sec:exp.}

\paragraph{Gridworld}\label{sec:exp_gw}
We use gridworlds in the first set of experiments. Each grid may have a different "terrain" type such that passing the grid will incur certain reward. For each grid,  a feature vector $\phi(s)$ denotes the terrain type, and the true reward can be expressed as $R^* =  w^* \cdot \phi(s)$. The agent's goal is to move to a {goal} grid with {\it maximal} reward under the {\it worst possible} weights that are consistent with the expert.
  In other words, the maxmin policy is a safe policy, as it avoids possible negative side effects \cite{AmodeiOSCSM16}. In the experiments, we construct the expert policy $\pi_E$ in a small (10$\times$10) demonstration gridworld that contains a subset of the terrain types. One expert policy is provide, and the number of trajectories that we need to estimate the expert policy's cumulative feature follows the sample complexity analysis as in~\cite{SyedS08}. In the following experiment we set $\epsilon=0.5$ which defines $P_R$ and captures how close to optimal the expert is. 

An example behavior is shown in Figure \ref{fig:gw}. There are 5 possible terrain types. The expert policy in Figure~\ref{fig:gw} (left) has only seen 4 terrain types. We compute the maxmin policy in the "real-world" MDP of a much larger size (50$\times$50) with all 5 terrain types using Algorithm~\ref{alg:FPL maxmin} with the reward polytope $P_R$ implicitly specified by the expert policy. Figure~\ref{fig:gw} (middle) shows that our maxmin policy avoids the red-colored terrain that was missing from the demonstration MDP.  To facilitate observation, Figure~\ref{fig:gw} (right) shows the same behavior by an agent trained in a smaller MDP. Figure~\ref{fig:dtgw_occu} compares the maxmin policy to a baseline. The baseline policy is computed in an MDP whose reward weights are the same as the demonstration MDP for the first four terrain types and the fifth terrain weight is chosen at random. Our maxmin policy is much safer than the baseline as it completely avoids the fifth terrain type. It also imitates the expert's behavior by favoring the same terrain types.

\begin{wrapfigure}{r}{0.5\textwidth}
	\vspace{-10pt}
	\centering
	\includegraphics[width = 0.45 \textwidth]{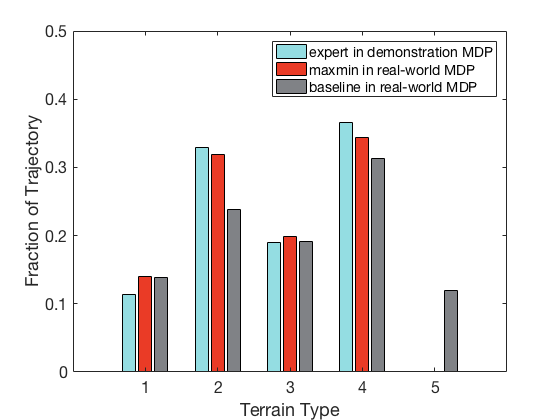}
	\caption{Experiment results comparing our maxmin policy to a baseline. The baseline was computed with a random reward for the fifth terrain and the other four terrain rewards set the same as the demonstration MDP. Our maxmin policy is much safer than the baseline as it completely avoids traversing the fifth (unknown) terrain type. It should also be noticed that the maxmin policy learns from the expert policy while achieving the goal of avoiding potential negative side effects, as the fraction of trajectory of each terrain type closely resemble the expert.  }
	\label{fig:dtgw_occu}
\end{wrapfigure}

We also implemented the maxmin method in gridworlds with a stochastic transition model. The maxmin policy (see Figure~\ref{fig:gw-stoc-1} in Section~\ref{sec:suppexp} of the supplementary material) is more conservative comparing to the deterministic model, and chooses paths that are further away from any unknown terrains. More details and computation time can be found in the supplementary material.

It should be noted that the training and testing MDPs are different. More specifically, the red terrain type is missing from the expert demonstration, and the testing MDP is of a larger size. As discussed in the Introduction, our formulation allows the testing MDP in which the agent operates to be different from the training MDP in which the expert demonstrates, as long as the two MDPs share the same feature space. All of our experiments have this property. To the limit of our knowledge, apprenticeship learning requires the training and testing MDPs to be the same, thus a direct comparison is not possible. For example, in the gridworld experiments, one has to explicitly assign a reward to the "unknown" feature in order to apply apprenticeship learning, which may cause the problem of reward misspecification and negative side effects. Our maxmin solution is robust to such issues. 

\begin{figure}[h]
\centering
	\hspace*{\fill}%
	\begin{subfigure}\centering
		\includegraphics[width=0.22\textwidth]{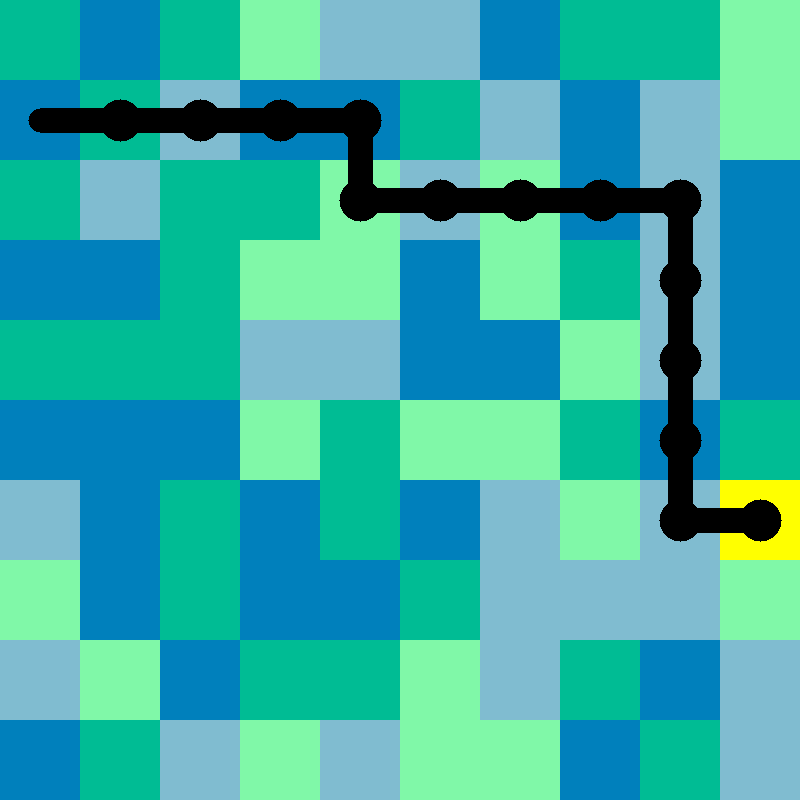}\hfill
	\end{subfigure}
	\begin{subfigure}\centering
		\includegraphics[width=0.22\textwidth]{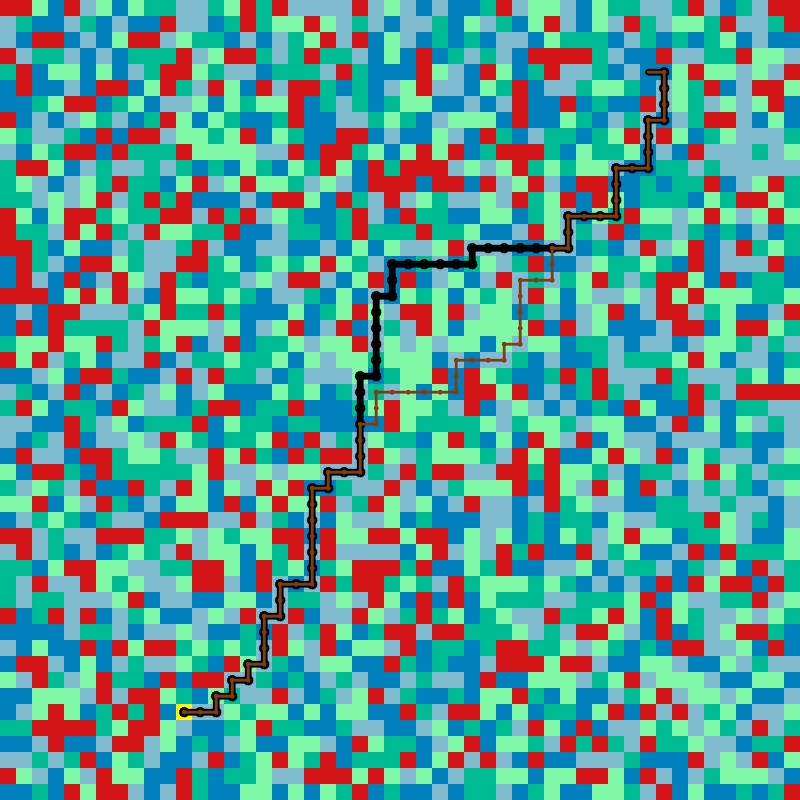}
	\end{subfigure}\hfill
	\begin{subfigure}\centering
		\includegraphics[width=0.22\textwidth]{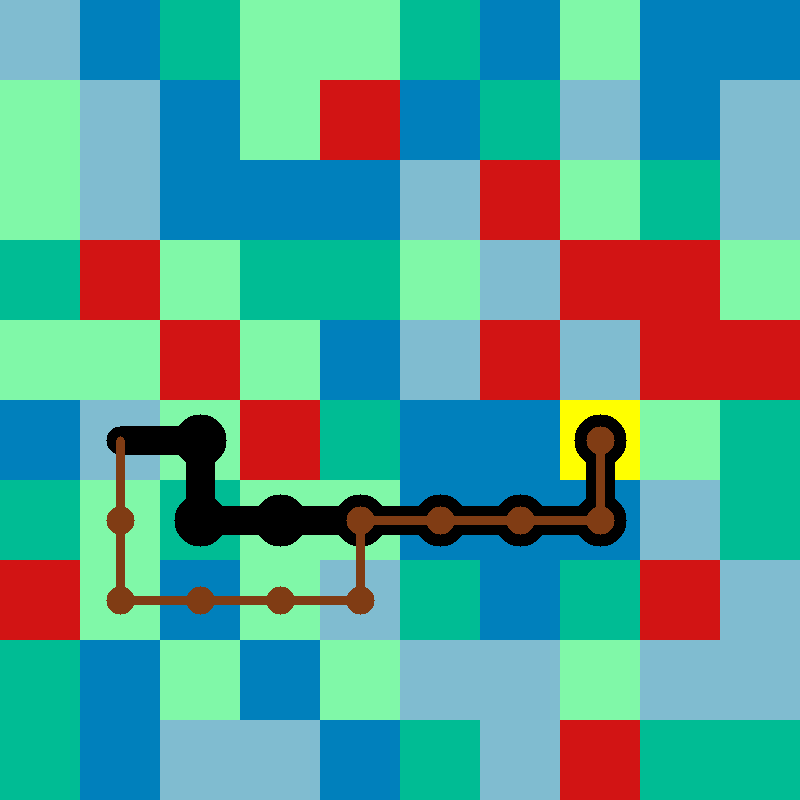}
	\end{subfigure}
	\hspace*{\fill}%
	\caption{An example of maxmin policy in gridworlds.  {\bf Left}: an expert policy in the small demonstration MDP, where 4 of 5 terrain types were used and their weights were randomly chosen. The expert policy guides moving towards the yellow goal grid while preferring the terrains with higher rewards (light blue and light green).  {\bf Middle}: when faced with terrain types (red-colored) that the expert policy never experienced, maxmin policy avoids such terrains and the accompanying negative side effects. The agent learns to operate in a larger (50$\times$50) grid world. 
	{\bf Right}: an agent in a smaller MDP to facilitate observation.
	The maxmin policy generates two possible trajectories.  }
\label{fig:gw}
\end{figure}

\begin{wrapfigure}{r}{0.4\textwidth}
	\centering
	\includegraphics[width = 0.3\textwidth]{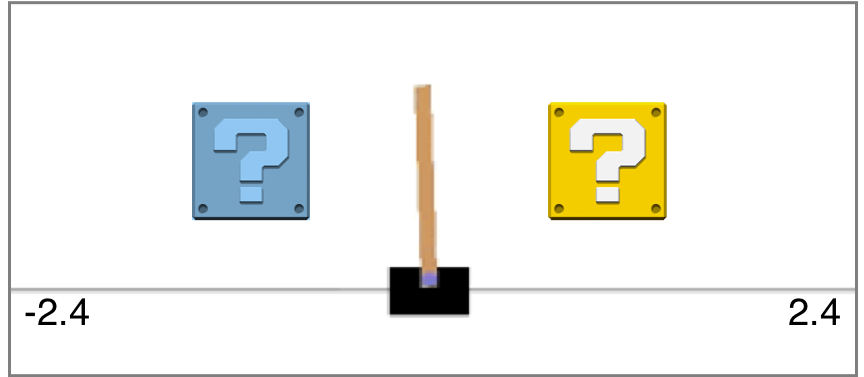}
	\caption{Modified cartpole task with two additional features -- questions blocks to either side of the center. The rewards associated with passing these blocks are not provided to the agent.}
	\label{fig:cartpole_env}
\end{wrapfigure}

\paragraph{CartPole}Our next experiments are based on the classic control task of cartpole and the environment provided by OpenAI Gym \cite{BrockmanCPSSTZ16}.  While we can only solve the problem approximately using model-free learning methods, our experiments show that our FPL-based algorithm can learn a safe policy efficiently for a continuous task. Moreover, if provided with more expert policies, our maxmin learning method can easily accomodate and learn from multiple experts. 

\begin{figure}
	\centering
	\begin{minipage}{0.48\textwidth}
		\centering
		\includegraphics[width = \textwidth]{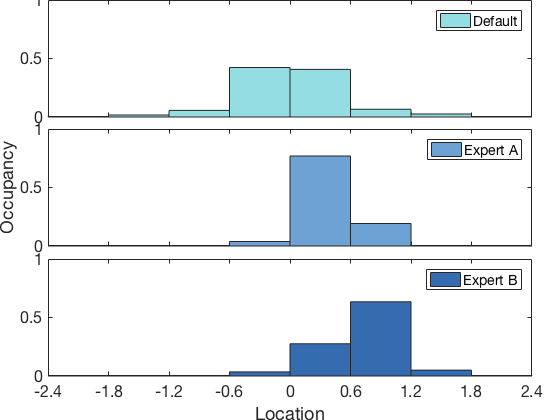}
		\caption{Behavior examples of different policies. Occupancy is defined as the number of steps appearing at a location divided by the total steps. {\bf top}: In the default setting without any question blocks, the travel range is relatively symmetric around the center of the field. {\bf mid}: In the presence of the blue question block to the left, an expert policy $A$ guides movements to the right. {\bf bottom}: In {\it scenario B} where only the yellow question block is present, expert policy $B$ also guides movement to the right. }
		\label{fig:cp_basic}
	\end{minipage}
	\hfill
	\begin{minipage}{0.48\textwidth}
		\centering
		\includegraphics[width = \textwidth]{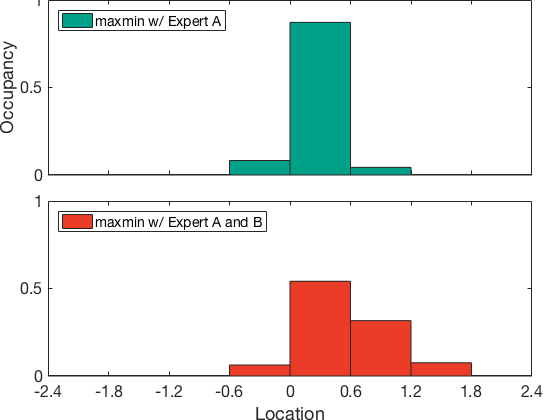}
		\caption{Maxmin policy learnt with different expert policies. {\bf top}: Given {\it Expert A} policy only, the agent learns to stay within a  narrow range near slightly right to the center to avoid both question blocks. Because the agent has no knowledge about the yellow block, a maxmin policy avoids it.  {\bf bottom}: When given both {\it Expert A} and {\it Expert B} policies, the agent learns that it is safe to pass the yellow block, so the range is wider and extends more to the right comparing to the maxmin policy learnt from {\it Expert A} alone.   }
		\label{fig:cp_maxmin}
	\end{minipage}
	\vspace{-10pt}
\end{figure}

We modify the cartpole problem by adding two possible features to the environment as the two question blocks shown in Figure~\ref{fig:cartpole_env}, and more details in the supplementary material. The agent has no idea of what consequences passing these two blocks may have. Instead of knowing the rewards associated with these two blocks, we have expert policies from two other related scenarios. The first expert policy ({\it Expert A}) performs well in {\it scenario A} where only the blue block to the left of the center is present, and the second expert policy ({\it Expert B}) performs well in {\it scenario B} where only the yellow block to the right of the center is present.  The behavior of expert policies in a default scenario (without any question blocks),  and scenarios $A$ and $B$ are shown in Figure~\ref{fig:cp_basic}. It is obvious that comparing with the default scenario, the expert policies in the other two scenarios prefer to travel to the right side. Intuitively, it seems that the blue block incurs negative effects while the yellow block is either neutral or positive.


Now we train the agent in the presence of both question blocks. First, we provide the agent with {\it Expert A} policy alone, and learn a maxmin policy. The maxmin policy's behavior is shown in Figure~\ref{fig:cp_maxmin} ({\bf top}). It tries to avoid both question blocks since it observes that Expert A avoids the blue block and it has no knowledge of the yellow block. Then, we provide both $Expert \ A$ and $Expert \ B$ to the agent, and the resulting maxmin policy guides movement in a wider range extending to the right of the field as shown in Figure~\ref{fig:cp_maxmin} ({\bf bottom}). This time, our maxmin policy also learns from {\it Expert B} that the yellow block is not harmful.

The experiment demonstrates that our maxmin method works well with complex reinforcement learning tasks where only approximate MDP solvers are available.

%% file: conclusion.tex
\section{Discussion}

In this paper, we provided a theoretical treatment of the problem of reinforcement learning in the presence of mis-specifications of the agent's reward function, by leveraging data provided by experts.  The posed optimization can be solved exactly in polynomial-time by using the ellipsoid methods, but a more practical solution is provided by an algorithm which takes a follow-the-perturbed-leader approach. Our experiments illustrate the fact that this approach can successfully learn robust policies from imperfect expert data, in both discrete and continuous environments. It will be interesting to see whether our maxmin formulation can be combined with other methods in RL such as hierarchical learning to produce robust solutions in larger problems.

%% file: appendix.tex
\section*{Supplementary Material}

\section{Missing Proofs from Section~\ref{sec:alg}}\label{sec:missingproof}

\paragraph{The Ellipsoid Method} The following theorem, reworded from~\cite{Khachiyan79,GrotschelLS81, KarpP82}, states that given a separation oracle of a convex polytope, the ellipsoid method can optimize any linear function over the convex polytope in polynomial time.

\begin{theorem}[Ellipsoid Method]\label{thm:ellipsoid}(\cite{Khachiyan79,GrotschelLS81,KarpP82}) Let $P$ be a $d$-dimensional closed, convex subset of $\mathbb{R}^d$ defined as the intersection of finitely many halfspaces, and $SO$ be a {poly-time} separation oracle for $P$. Then it is possible to
find an element in $\argmax_{{x} \in P} \{{c} \cdot {x}\}$ for any ${c} \in \mathbb{R}^d$ (i.e. solve linear programs) in time polynomial in $d$ and $\langle P \rangle$ using the ellipsoid method, if $P$ can be described implicitly
using $\langle P\rangle$ bits.~\footnote{We say a polytope $P$ can be described implicitly using $\ell$ bits if there exists a description of the polytope $P$ such that all constraints only use coefficients with bit complexity $\ell$.}

\end{theorem}

\begin{prevproof}{Lemma}{lem:SO_f}
	Lemma~\ref{lem:bit complexity of P_F}  shows that $P_F$ can be implicitly described using $\langle P_F\rangle = \poly(\langle M\setminus R \rangle,\langle \phi \rangle,k)$ bits. Maximizing any linear function $w\cdot \mu$ can be solved by querying $\ALG$ on MDP $M\setminus R$ with reward function $w\cdot \phi(\cdot)$. Since MDP $\left (M\setminus R, w\cdot \phi(\cdot)\right)$ has bit complexity polynomial in $\langle M\setminus R \rangle$, $\langle \phi \rangle$, $k$, and $\langle w \rangle$,  we can solve the linear optimization problem in time $\poly(\langle P_F\rangle,k, \langle w \rangle)$. By Lemma~\ref{lem:sep-opt}, we can solve the separation problem in time $\poly(\langle P_F\rangle, k, \langle y \rangle)$ on any input $y\in \mathbb{R}^k$. Hence, we can design a polynomial time separation oracle.
\end{prevproof}

\begin{lemma}\label{lem:bit complexity of P_F}
Polytope $P_F$ for any MDP without reward function $M\setminus R$ can be implicitly described using $\poly(\langle M\setminus R \rangle,\langle \phi \rangle,k)$ bits.
\end{lemma}
\begin{proof} 
The following constraints explicitly describe all $\mu \in P_F$, where $x_{sa}$s correspond to the occupancy measure of some policy $\pi$. 
\begin{align*}
	&\mu = \sum_{s} \phi(s) \sum_{a} x_{s a} &\\
& \sum_a x_{sa} = \Pr(s_0 = s) + \gamma \sum_{s',a}x_{s'a} P_{sa} \qquad \forall s &\\	
& x_{sa} \geq 0 &
\end{align*}
Our statement follows from the fact that all the coefficients in these constraints have bit complexity $\langle M\setminus R \rangle$ or $\langle \phi \rangle$.
\end{proof}

\paragraph{Intuition behind Lemma~\ref{lem:convhull}} The intuition behind Lemma~\ref{lem:convhull} is that the separation oracle $SO_F$ tries to search over all possible weights $w$ to find one to separate the query point $\mu$ from  $P_F$ using the ellipsoid method. Along the way, it queries a set of weights (this is our set $C$) on $\ALG$ trying to find a separating weight $w$ such that $\mu\cdot w>\mu_w \cdot w$. If such a separating weight is found, $SO_F$ terminates immediately and outputs ``\textsc{No}'' together with the corresponding separating hyperplane. The $SO_F$ says ``\textsc{Yes}'' only when it has searched over a polynomial number of weights and concludes that there is no possible weight to separate $\mu$. The reason that $SO_F$ can draw such a conclusion is due to the ellipsoid method. In particular, when $SO_F$ says ``\textsc{Yes}'',  the correctness of the ellipsoid algorithm implies that $\mu$ is in the convex hull of all the extreme points of $P_F$ that have been outputted by the $\ALG$.

\paragraph{Follow-the-Perturbed-Leader} Kalai and Vempala~\cite{KalaiV05} proposed the FPL algorithm and showed that in expectation, the regret is small against any oblivious adversary.~\cite{HutterP05} showed that the same regret bound extends to settings with adaptive adversary. To obtain a high probability bound, one can construct a martingale to connect the actual reward and the expected reward obtained by the agent, then apply the Hoeffding-Azuma inequality.

\begin{theorem}[Follow-the-Perturbed-Leader]~\cite{KalaiV05,HutterP05, CL}\label{thm:FPL}
	Let $d_1,\ldots,d_T$ be a sequence of decisions. Let $s_1,\ldots, s_T$ be a state sequence chosen by an adaptive adversary, that is, $s_t$ can be selected based on all the previous states $s_1,\ldots,s_{t-1}$ and all the previous decisions $d_1,\ldots, d_{t-1}$ for every $t\leq T$. If we let $d_t$ be $\argmax_{d\in D} d\cdot \left( \sum_{i=1}^{t-1} s_{i}+p_t\right)$, where $p_t$ is drawn uniformly from $[0,1/\delta]^n$ for some $\delta>0$,  then 
	$$\E\left[\sum_{t=1}^T d_t\cdot s_t-\max_{d\in D} \sum_{t=1}^T d\cdot s_t\right]\geq -\delta\cdot C_1C_2 T-\frac{2C_3}{\delta}.$$ $C_1$ is an upper bound of $||s||_1$ for all $s\in \bS$, $C_2$ is an upper bound of $|d\cdot s|$ for all $d\in D$ and $s\in \bS$, and $C_3$ is an upper bound of $||d||_1$ for all $d\in D$. Moreover, for all $\xi\geq 0$, with probability at least $1-\xi$, the actual accumulative reward under any adaptive adversary satisfies, $$\sum_{t=1}^T d_t\cdot s_t-\max_{d\in D} \sum_{t=1}^T d\cdot s_t\geq -\delta\cdot C_1C_2 T-\frac{2C_3}{\delta}-2C_2\sqrt{T\ln \frac{1}{\xi}}.$$ 
\end{theorem}

Using Theorem~\ref{thm:FPL}, we are ready to prove Theorem~\ref{thm:FPL maxmin}.

\begin{prevproof}{Theorem}{thm:FPL maxmin}
	We use $P$ to denote the sequence $p_1,\ldots, p_T$ and $Q$ to denote the sequence $q_1,\ldots, q_T$. First, notice that every realization of $Q$ defines a deterministic adaptive adversary for the agent. In the setting of Algorithm~\ref{alg:FPL maxmin}, we can take $C_1$ to be $k$, $C_2$ to be $k^2$, and $C_3$ to be $k$. By Theorem~\ref{thm:FPL} (Section~\ref{sec:missingproof} of the supplementary material), we know that for all $\xi\geq 0$, $\Pr_{P\sim U[0,1/\delta]^{kT}} [\sum_{t=1}^T \mu_t\cdot w_t-\max_{\mu\in P_F} \sum_{t=1}^T \mu\cdot w_t \geq -k^2\sqrt{T}(3+2\sqrt{\ln 1/\xi}) | Q]\geq 1-\xi$ for every realization of $Q$. Similarly, every realization of $P$ also defines a deterministic adaptive adversary for the designer, and by Theorem~\ref{thm:FPL} , we know that $\Pr_{Q\sim U[0,1/\delta]^{kT}} [-\sum_{t=1}^T \mu_t\cdot w_t+\min_{w\in P_R} \sum_{t=1}^T \mu_t\cdot w \geq -k^2\sqrt{T}(3+2\sqrt{\ln 1/\xi})| P]\geq 1-\xi$ for any realization of $P$. Let $B= k^2\sqrt{T}\left(3+2\sqrt{\ln 1/\xi}\right)$. By the union bound, with probability at least $1-2\xi$ over the randomness of $P$ and $Q$  \begin{equation}\label{eq:agent regret}
\sum_{t=1}^T \mu_t\cdot w_t-\max_{\mu\in P_F} \sum_{t=1}^T \mu\cdot w_t \geq -B
  \end{equation}
  and 
\begin{equation}\label{eq:designer regret}
-\sum_{t=1}^T \mu_t\cdot w_t+\min_{w\in P_R} \sum_{t=1}^T \mu_t\cdot w \geq -B
\end{equation}
	
	Next, we argue that $\frac{1}{T}\cdot \sum_{t=1}^T \pi_t$ is an approximate maxmin policy.
	\begin{align*}
	\min_{w\in P_R}\sum_{t=1}^T \mu_t\cdot w&\geq \sum_{t=1}^T \mu_t\cdot w_t	-B & (\text{Eq.~\eqref{eq:designer regret}})\\
	&\geq \max_{\mu\in P_F} \sum_{t=1}^T \mu\cdot w_t-2B & (\text{Eq.~\eqref{eq:agent regret}})\\
	&\geq T\cdot \max_{\mu\in P_F}\min_{w\in P_R} \mu\cdot w-2B &
	\end{align*}
The last inequality is because that on the LHS (line 2) the designer is choosing a fixed strategy $\frac{1}{T}\cdot\sum_{t=1}^T w_t$, while on the RHS (line 3) the designer can choose the worst possible strategy for the agent. Therefore, if the agent uses policy $\frac{1}{T}\cdot \sum_{t=1}^T \pi_t$, it guarantees expected reward $\max_{\mu\in P_F}\min_{w\in P_R} \mu\cdot w-2B/T$.

Finally, in every iteration $t$, we query $\ALG$ once to compute $\pi_t$ and $\mu_t$, and we use the ellipsoid method to find $w_t$ using $O(k^2 (\langle M\backslash R \rangle^2+(\log T)^2))$ queries to $SO_R$ and poly$(k, \langle M\backslash R \rangle, \log T)$ regular computation steps. During each query, $SO_R$ calls $\ALG$. Thus, our result is a reduction from the maxmin learning problem to simply solving an MDP under given weights. Any improvement on $\ALG$ will also improve the running time of Algorithm~\ref{alg:FPL maxmin}. We discuss the empirical running time in section~\ref{sec:suppexp} of the supplementary material.
\end{prevproof}

\input{model_free_approx_solver}

\section{Experiment Details}\label{sec:suppexp}
In every iteration of Algorithm~\ref{alg:FPL maxmin} and Algorithm~\ref{alg:eps FPL maxmin}, step 5 computes a minimizing weight in $P_R$. Instead of using the ellipsoid method to solve the LP, we use the analytic center cutting-plane method (see~\cite{BoydV07} for a brief overview) throughout our experiments. The method combines good practical performance with reasonable simplicity. 

\subsection{Gridworld}
The domain contains five types of terrain. Four terrain types are used in the demonstration gridworld where we construct the expert policy. We select the rewards for these four terrain types uniformly from $[-0.5,0]$, and the target has a reward of 10. The reward of each terrain type is deterministic. The demonstration MDP is uniformly composed of four terrain types, 25\% each type. The fifth terrain type (red colored as in Figure~\ref{fig:gw}) is not present in the demonstration gridworld. The agent is trained in a "real-world" MDP that is composed uniformly of all five terrain types, 20\% each. We select maps that are feasible, such that for all rewards in the consistent reward polytope, value iteration has a solution for the agent to reach the goal. We use feature vectors that indicate the terrain type of each state, choose a discount factor of 0.95, and use value iteration throughout the experiment. The consistency between the expert policy and the reward function is defined with $\epsilon = 0.5$. 

\paragraph{Deterministic transition model} In an MDP with deterministic transition model, the agent moves in exactly the direction chosen by the agent. We run FPL for $5000$ iterations and use the average of policies output by the last $2500$ iterations as the maxmin policy. Figure~\ref{fig:dtgw_occu} shows that our maxmin policy is much safer than a baseline. The baseline policy is computed in an MDP whose reward weights are the same as the demonstration MDP for the first four terrain types and the fifth terrain weight is chosen uniformly at random from [-1,0]. The expert policy for the displayed results is constructed by computing the optimal policy in an demonstration MDP with rewards for the first four terrain type set as $[-0.5, -0.2, -0.4, -0.1]$. The results are accumulated from 100 individual runs using the same expert policy. Examples of the baseline trajectories are shown in Figure~\ref{fig:gw-1}. 

\begin{figure}[ht]
\begin{center}
	\includegraphics[width=0.2\textwidth]{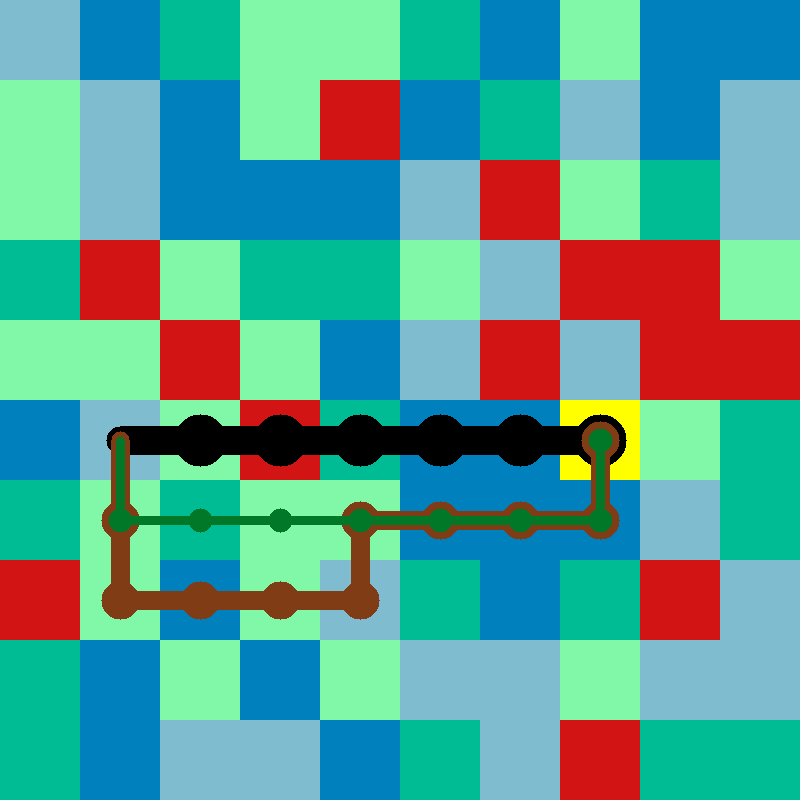}\hfill
\end{center}
\caption{Trajectories chosen by policies generated using weights randomly assigned to the red-colored unknown feature. Although this feature may have negative side effects, the random agent may still go through it. }
\label{fig:gw-1}
\end{figure}

\paragraph{Stochastic transition model} At each state, there is 10\% chance that the agent will go in a random direction regardless of the action chosen by the agent. The agent will receive rewards based on the state it actually lands in. We show in Figure~\ref{fig:gw-stoc-1} that to mitigate the higher risk of traversing the unknown terrain type, our maxmin policy appears to be more conservative than the deterministic case. Although Figure~\ref{fig:stocgw_occu} shows that it cannot absolutely avoid the unknown terrain type due of the stochastic nature of the model, the percentage is much lower than the baseline. The baseline was computed with the same reward weights as in the deterministic case.

 \begin{figure}
\begin{center}
	\includegraphics[width=0.2\textwidth]{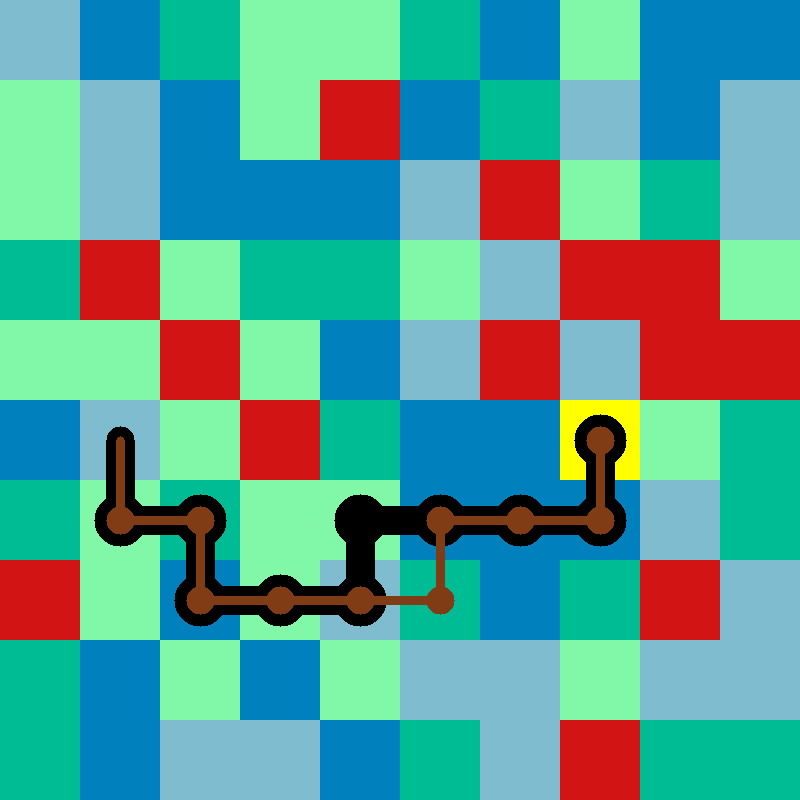}\hfill
\end{center}
\caption{At each state, there is $10\%$ chance that the agent will go in a random direction irrespective of the action chosen by the agent,  our maxmin method is still valid. Comparing to Figure~\ref{fig:gw} ({\bf right}), the maxmin policy also avoids going to the peripheral of the red-colored unknown feature.  }
\label{fig:gw-stoc-1}
\end{figure}

 \begin{figure}\centering
	\includegraphics[width = 0.5 \textwidth]{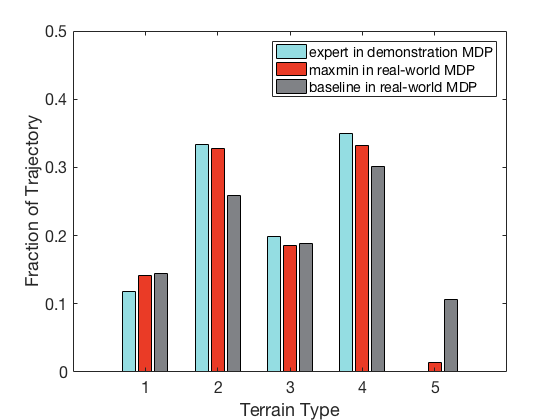}
	\caption{In the gridworld with stochastic transition model, our maxmin policy has a small chance of traversing the unknown terrain type disregard of being more conservative than the maxmin policy in the deterministic case. The percentage is much lower than the baseline.  }
	\label{fig:stocgw_occu}
\end{figure}

\paragraph{Computation Performance}
In our grid world experiment, the worst case running time of ALG is $O(n^2)$, but experiments show a more benign runtime of $O(n^{1.5})$. For a $50\times50$ grid world with 25 features, Algorithm~\ref{alg:FPL maxmin} appears to converge after 325 iterations of FPL with total runtime of 3324 seconds (average of 20 trials, ordinary desktop computer). 
Instead of using the ellipsoid method, we used analytic center cutting-plane method, and the running time appears to scale in the order of $O(k^2)$.

 \subsection{CartPole}
 We modify the classic CartPole task in the OpenAI Gym environment by adding features that may incur additional rewards. This is represented by the question blocks in Figure~\ref{fig:cartpole_env}. The two question blocks correspond to feature indicators for the agent's horizontal position in the range of $[-1.2, 0)$ and $[0.6, 1.8)$. We keep the same episode termination criterions for the pole angle and cart position as the original environment. An episode is considered ending without failing if the pole angel and cart position meet the criterion and the episode length is greater than 500. The agent receives a reward of $+1$ for surviving every step. 
 
 We use longer episodes than the original problem to allow more diverse movement, while it also makes the task more challenging. During validation of a policy, we consider the task solved as getting a target average reward over 100 consecutive episodes with less than five failed episodes. The target average reward depends on the reward we assign for passing the question blocks. If each step spent at question block $i$ incurs reward of $r_i$, the target average reward is set to be $450+25\sum r_i$. For example, in $scenario \ A$, only the blue question block exists and it incurs reward of $-2$, our expert policy $Expert \ A$ passes the validation criterion by getting average reward higher than 400 over 100 consecutive episodes with less than 5 failed episodes. Indeed, our $Expert \ A$ policy performs quite well by getting a reward greater than 450 in $scenario \ A$. In $scenario \ B$, only the yellow question block is present and it incurs reward of $+2$. $Expert \ B$ passes the validation criterion with reward greater than 1000. 
 
 The agent is in an MDP with both blue and yellow question blocks whose reward polytopes are implicitly defined by the expert policy. We use Q-learning and apply updates using minibatches of stored samples as the MDP solver. Notice that for this problem, our MDP solver is not necessarily optimal. We computed maxmin policies when provided with different expert policies. The results in Figure~\ref{fig:cp_maxmin} are from testing the maxmin policy for $2000$ episodes.

%% file: model_free_approx_solver.tex
\section{Maxmin Learning using an Approximate MDP Solver}\label{sec:approx FPL}


In the previous sections, we assume that we have access to an MDP solver $\ALG$ that solves any MDP $M$ optimally in time polynomial in $\langle M\rangle$. However, in practice, solving large-size MDPs, e.g. continuous control problems, exactly could be computationally expensive or infeasible. Our FPL-based algorithm also works in cases where we can only solve MDPs approximately. 


{Suppose we are given access to an additive FPTAS $\WALG$ for solving MDPS. More specifically, $\WALG$ finds in time polynomial in $\langle M\rangle, 1/\eta$ a solution $(\pi^*_{\eta}, \mu^*_{\eta})$, such that $\E_{s_0\sim D}[V^{\pi^*_{\eta}}(s_0) | M]\geq \max_\pi \E_{s_0 \sim D}[V^\pi(s_0) | M] - \eta $. Notice that  the weights of $M$'s reward function have $L_1$-norm $L$.}

We face two challenges when we replace $\ALG$ with $\WALG$: (i) we can no longer find the best policy $\mu_t$ with respect to all the previous weights plus the perturbation in every iteration, and (ii) we no longer have a separation oracle for $P_R$, as the $SO_R$ (Algorithm~\ref{alg:SO}) relies on the MDP solver when $P_R$ is implicitly specified by the expert's policy. It turns out (i) is not hard to deal with, as the FPL algorithm is robust enough to work with only an approximate leader. (ii) is much more subtle. We design a new algorithm and use it as a proxy for the polytope $P_R$. We call this new algorithm a weird separation oracle (following the terminology in~\cite{CaiDW13a}) as the points it may accept do not necessarily form a convex set, even though it does accept all points in $P_R$. It may seem at first not clear at all why such a weird separation oracle can help us. However, we manage to prove that just with this weird separation oracle, we can still compute an approximate minimizing weight vector $w_t$ in $P_R$ in every iteration (Step 5 of Algorithm~\ref{alg:FPL maxmin}). Combining this with our solution for challenge (i), we can still compute an approximately maxmin policy with essentially the same performance as in Algorithm~\ref{alg:FPL maxmin}. 

\begin{theorem}\label{thm: approx FPL maxmin}
{If we replace the exact MDP solver $\ALG$ with an approximate solver $\WALG$ in step 4 of Algorithm ~\ref{alg:FPL maxmin},} then for any $\xi\in (0,1/2)$ and any $c>0$, with probability at least $1-2\xi$, Algorithm~\ref{alg:FPL maxmin} finds a policy $\pi$ after $T$ rounds of iterations such that its expected reward under any weight from $P_R$ is at least $\max_{\mu\in P_F}\min_{w\in P_R} \mu\cdot w-\frac{k^2\left(6+4\sqrt{\ln 1/\xi}\right)}{\sqrt{T}} - 2c $. In every iteration, Algorithm~\ref{alg:FPL maxmin} makes one query to $\WALG$ and a polynomial number of queries to $SO_R$. In particular, for every query to $\WALG$, we first divide the input by $2T$ then feed it to $\WALG$ and ask for a policy that is at most $c/2T$ worse than the optimal one.
\end{theorem}



The proof of Theorem~\ref{thm: approx FPL maxmin} is similar to the proof of Theorem~\ref{thm:FPL maxmin}. We use the bounds provided by Lemma~\ref{lem:approximateFPL} instead of Theorem~\ref{thm:FPL}, and change the RHS in Equation~\eqref{eq:agent regret} from $-B$ to $-k^2\sqrt{T}\left(3+2\sqrt{\ln 1/\xi}\right) -2cT$ accordingly. The rest of the proof remains the same. 

Assume we have a procedure $M_\eta$ for $\eta$-approximating linear programs over the decision set $D$ such that for all $s\in \mathbb{R}^k$, 
\[	s \cdot M_\eta(s) \geq \argmax_{d\in D} s \cdot d - \eta.\]

\begin{lemma}[Follow the Approximate Perturbed Leader]~\cite{Ben-TalHKM15}\label{lem:approximateFPL} Let $d_1,\ldots,d_T$ be a sequence of decision made by an $\eta$-approximating procedure $M_\eta$ such that $d_t = M_\eta(\sum_{i=1}^{t-1} s_i + p_t)$. Then
$$\E\left[\sum_{t=1}^T d_t\cdot s_t-\max_{d\in D} \sum_{t=1}^T d\cdot s_t\right]\geq -\delta\cdot C_1C_2 T-\frac{2C_3}{\delta} - 2\eta T .$$ The definition of constants $C_1$, $C_2$ and $C_3$ are the same as in Theorem~\ref{thm:FPL}.
Moreover, for all $\xi\geq 0$, with probability at least $1-\xi$, the actual accumulative reward under any adaptive adversary satisfies, \begin{align*}
&\sum_{t=1}^T d_t\cdot s_t-\max_{d\in D} \sum_{t=1}^T d\cdot s_t\geq\\
 &~~~~~~~~~~~~~~~~~~ -\delta\cdot C_1C_2 T-\frac{2C_3}{\delta}-2C_2\sqrt{T\ln \frac{1}{\xi}} - 2\eta T .
 \end{align*}
 \label{lemma:approx FPL}
\end{lemma}

An astute reader may have noticed that in the analysis above, we used the same separation oracle $SO_R$ as in section~\ref{sec:reward SO}. However, in the case when the separation oracle for the reward polytope is implicitly specified by an expert policy, $SO_R$ queries the MDP solver in step 1 of algorithm~\ref{alg:SO}. If we do not have an exact MDP solver $\ALG$, it is not clear how we can define a separation oracle for polytope $P_R$. We use Algorithm~\ref{alg:WSOR} as an proxy to polytope $P_R$.


\begin{algorithm}[ht]
\begin{algorithmic}[1]
\STATE Let $\mu_{w'}^{(\eta)}:=\WALG(w',\eta) $. It is the feature vector of the policy computed by the approximate MDP solver $\WALG$ with accuracy $\eta$, and $\mu_{w'}^{(\eta)}\cdot w' \geq \max_{\mu\in P_R} \mu\cdot w' - \eta$. 
\IF {$\mu_{w'}^{(\eta)}\cdot w'> \mu_E \cdot w' + \epsilon$} \STATE output ``NO'' , and 
	 $ \left(\mu_E- \mu_{w'}^{(\eta)} \right)\cdot w + \epsilon\geq 0$ as the separating hyperplane, 
	 since for all $w \in P_R, \mu_E\cdot w \geq \mu_{w'} ^{(\eta)}\cdot w -\epsilon$.
\ELSE \STATE output ``YES''.
\ENDIF
\end{algorithmic}
\caption{{\sf Weird Separation Oracle $WSO^\eta_R$ for the reward polytope $P_R$}}
\label{alg:WSOR}
\end{algorithm}

We call $WSO^\eta_R$ a weird separation oracle for for the reward polytope for $P_R$, because the set of $w'$ that it will accept is not necessarily convex. For example, the following may happen. First, we query two points $w_1$ and $w_2$ that are close to each other. Both are accepted by $WSO_R^\eta$, and it happens to be the case that $\WALG(w_1,\eta)$ and $\ALG(w_2,\eta)$ are both $\eta$ away from the optimal solutions. Now we query $w_3 = (w_1+w_2)/2$, and run $WSO^\eta_R$. Luckily (or unfortunately) $\WALG(w_3,\eta)$ is close to optimal, and $w_3$ is rejected.  

\begin{lemma}\label{lem:WSOR}
	For any linear optimization problem, we can construct a polynomial time algorithm based on the ellipsoid-method that queries $WSO_R^{\eta}$, such that it finds a solution that is at least as good as the best solution in polytope ${P = \{w | w \cdot \mu_E \geq w\cdot  \WALG(w',\eta)-\epsilon, \forall w' , \|w'\|_1 } \leq L\}$, although our solution does not necessarily lie in $P$. 
	\end{lemma}
\begin{prevproof}{Lemma}{lem:WSOR}
We only sketch the proof here. Solving a linear optimization can be converted into solving a sequence of feasibility problems by doing  binary search on the objective value. We show that for any objective value $\alpha$, as long as there is a solution $x\in P$ whose objective value $c\cdot x\geq \alpha$, our algorithm also finds a solution $x'$ such that $c\cdot x'\geq \alpha$. First, imagine we have a separation oracle for $P$, and the ellipsoid method needs to run $N$ iterations to determine whether there is a solution in $P$ whose objective value is at least $\alpha$. The correctness of ellipsoid method guarantees that if it hasn't found any solution after $N$ iterations, then the intersection of the  halfspace $c\cdot x\geq \alpha$ and $P$ is empty. The reason is that if the intersection is not empty it must have volume at least $r$, and the ellipsoid method maintains an ellipsoid that contains the intersection of the halfspace $c\cdot x\geq \alpha$ and $P$ and shrinks the volume of the ellipsoid in every iteration. After $N$ iterations the ellipsoid already has volume less than $r$.

Our algorithm also runs the ellipsoid method for $N$ iterations. In each iteration, we first check the constraint $c\cdot x\geq \alpha$, if not satisfied, we output this constraint as the separating hyperplane. If it is satisfied, instead of querying the real separation oracle for $P$, we query $WSO^\eta_R$. If the answer is ``YES", we have found a solution $x$ such that $c\cdot x\geq \alpha$. If the answer is ``NO", clearly this query point is not in $P$, and the outputted separating hyperplane contains the intersection of the halfspace $c\cdot x\geq \alpha$ and $P$. Therefore, whenever our algorithm accepts a point, it must have objective value higher than $\alpha$. Otherwise, the shrinking ellipsoid still contains the intersection of the halfspace $c\cdot x\geq \alpha$ and $P$. If our algorithm terminates after $N$ iterations without accepting point, we know that the intersection between the  halfspace $c\cdot x\geq \alpha$ and $P$ is empty as the volume of the ellipsoid after $N$ iterations is already too small.

\end{prevproof}

Consider the following three polytopes:

\noindent (i) $P_R := \left\{w\ |\ w\cdot\mu_E \geq w\cdot \mu -\epsilon,\ \forall \mu \in P_F\right \}$

\vspace{.1in}
\noindent (ii) $P = \{{w|w \cdot \mu_E \geq w\cdot \WALG(w',\eta) -\epsilon, \forall w' , \|w'\|_1 \leq L}\}$

\vspace{.1in}
\noindent (iii) $P_R^{(\epsilon + \eta)}:= \{w\ |\ w\cdot\mu_E \geq w\cdot \mu -\epsilon-\eta,  \forall \mu \in P_F\}$. 

\begin{fact}
	$P_R\subseteq P$.
	\label{fact:PRinP}
\end{fact}
\begin{fact}
	$WSO_R^\eta$ only accepts points that are in $P_R^{(\epsilon + \eta)}$.
	\label{fact:PRee}
\end{fact}
\begin{proof}
	Suppose $w\notin P_R^{(\epsilon+\eta)}$, then clearly $w\cdot \WALG(w,\eta)\geq \max_{\mu\in P_F} w\cdot \mu -\epsilon-\eta>w\cdot \mu_E$. Hence, $WSO_R^{\eta}$ will not accept $w$.
\end{proof}
\begin{lemma}
	For all $w$ in $P_R^{(\epsilon+\eta)}$,  $w\cdot\frac{\epsilon}{\epsilon+\eta}$ is in $P_R$.
	\label{lemma:shrinkw}
\end{lemma}
\begin{prevproof}{Lemma}{lemma:shrinkw}
From the definition of $P_R^{(\epsilon+\eta)}$, multiply both side of the inequality with $\frac{\epsilon}{\epsilon+\eta}$, and let $w' = w\cdot\frac{\epsilon}{\epsilon+\eta}$, $w'$ is in $P_R$.
\end{prevproof}

\begin{theorem}
For any $c>0$ and $\xi\in (0,1/2)$, with probability at least $1-2\xi$, Algorithm~\ref{alg:FPL maxmin} finds a policy $\pi$ after $T$ rounds of iterations such that its expected reward under any weight from $P_R$ is at least $\max_{\mu\in P_F}\min_{w\in P_R} \mu\cdot w-\frac{k^2\left(6+4\sqrt{\ln 1/\xi}\right)}{\sqrt{T}} - 4c$. In every iteration, Algorithm~\ref{alg:eps FPL maxmin} makes one query to $\ALG_\eta$ and a polynomial number of queries to Algorithm~\ref{alg:WSOR}. 
	\label{thm:FPL-WSOR}
\end{theorem}

Now, we are ready to describe the algorithm using only access to $WSO_R^{\eta}$.

\begin{algorithm}[ht]
\begin{algorithmic}[1]
\INPUT $T$: the number of iterations
\STATE Set $\delta:=\frac{1}{k\sqrt{T}}$, where $||w||_1\leq L$ for all $w\in P_R$. Set $\eta_1:=\frac{c}{2T}$ and $\eta_2:=\frac{c\epsilon}{2k^2T-c}$.
\STATE Arbitrarily pick some policy $\pi_1$ and compute $\mu_1 \in P_F$. Arbitrarily pick some reward weights $w_1$, and set $t = 1$.
\WHILE{$t\leq T$}{
\STATE Let policy $\pi_t$ and $\mu_t = \Psi(\pi_t)$ be the output of $\WALG\left( \left(\sum_{i=1}^{t-1} w_i+p_t\right )/T, \eta_1 \right)$ 
, where $p_t$ is drawn uniformly from $[0,1/\delta]^k$.
\STATE Use our algorithm in Lemma~\ref{lem:WSOR} with $WSO_R^{\eta_2}$ to solve $\mathop{\min} w^T (\sum_{i=1}^{t-1}\mu_t + q_t)$, where $q_t$ is drawn uniformly from $[0,1/\delta]^k$. Let $w'_t$ be the solution and set $w_t$ to be $w'_t\cdot\frac{\epsilon}{\epsilon+\eta_2}$.
\STATE $t:= t+1$.}
\ENDWHILE
\STATE Output the randomized policy $\frac{1}{T}\cdot \sum_{t=1}^T \pi_t$.
\end{algorithmic}
\caption{{\sf Finding the Maxmin Policy using Follow-the-Perturbed-Leader (FPL)}}
\label{alg:eps FPL maxmin}
\end{algorithm}

\begin{prevproof}{Theorem}{thm:FPL-WSOR}
At each time step $t$, using $WSO_R$,  Algorithm~\ref{alg:eps FPL maxmin} step 5 outputs a $w_t$. By Lemma~\ref{lem:WSOR} and Fact~\ref{fact:PRinP}, 
$$w'_t\cdot\left(\sum_{i=1}^{t-1}\mu_i + q_t\right) \leq \min_{w \in P_R}w\cdot\left(\sum_{i=1}^{t-1}\mu_i +q_t\right).$$
By Lemma~\ref{lemma:shrinkw} and Fact~\ref{fact:PRee}, 
\begin{align*}
	&w_t \cdot\left(\sum_{i=1}^{t-1}\mu_i + q_t\right) = \frac{\epsilon}{\epsilon+\eta_2}\cdot w_t' \cdot\left(\sum_{i=1}^{t-1}\mu_i + q_t\right)\\
	 \leq & \min_{w \in P_R }w\cdot\left(\sum_{i=1}^{t-1}\mu_i +q_t\right)
+ \frac{2k^2T}{\epsilon+\eta_2},
\end{align*}  where we used the fact that $$ -w_t\left(\sum_{i=1}^{t-1}\mu_i + q_t\right) \leq  2k^2T.$$
Since $c = \frac{2\eta_2 k^2T}{\epsilon+\eta_2}$, we can use Lemma~\ref{lemma:approx FPL} and replace the RHS in Equation~\eqref{eq:designer regret} that was used in the proof of Theorem~\ref{thm:FPL maxmin} to  $-k^2\sqrt{T}\left(3+2\sqrt{\ln 1/\xi}\right) - 2cT$. The analysis for $\mu_t$ remains the same as in Thoerem~\ref{thm: approx FPL maxmin}.
\end{prevproof}